\documentclass[Journal,twocolumn,10pt]{IEEEtran}

\usepackage{graphicx}
\usepackage{amsfonts}
\usepackage{amssymb}
\usepackage{amsmath}
\usepackage{fancyhdr}
\usepackage{lastpage}
\usepackage{subfigure}

\usepackage{algorithm}
\usepackage{algpseudocode}
\usepackage{algorithmicx}

\newcommand{\comment}[1]{}

\newtheorem{theorem}{Theorem}[section]
     \newtheorem{lemma}[theorem]{Lemma}

     \newcommand{\qed}{\nobreak \ifvmode \relax \else
           \ifdim\lastskip<1.5em \hskip-\lastskip
           \hskip1.5em plus0em minus0.5em \fi \nobreak
           \vrule height0.75em width0.5em depth0.25em\fi}
\def\keywords{\vspace{.5em}
{\textit{Keywords}:\,\relax%
}}
\def\endkeywords{\par}

\begin{document}

\title{ Large-Scale Clustering Based on Data Compression}

\author{\authorblockN{Xudong Ma \\}
\authorblockA{Pattern Technology Lab LLC, Delaware, U.S.A.\\
Email: xma@ieee.org} }

\maketitle

\begin{abstract}

This paper considers the clustering problem for large data sets. We
propose an approach based on distributed optimization. The
clustering problem is formulated as an optimization problem of
maximizing the classification gain. We show that the optimization
problem can be reformulated and decomposed into small-scale sub
optimization problems by using the Dantzig-Wolfe decomposition
method. Generally speaking, the Dantzig-Wolfe method can only be
used for convex optimization problems, where the duality gaps are
zero. Even though, the considered optimization problem in this paper
is non-convex, we prove that the duality gap goes to zero, as the
problem size goes to infinity. Therefore, the Dantzig-Wolfe  method
can be applied here. In the proposed approach, the clustering
problem is iteratively solved by a group of computers coordinated by
one center processor, where each computer solves one independent
small-scale sub optimization problem during each iteration, and only
a small amount of data communication is needed between the computers
and center processor. Numerical results show that the proposed
approach is effective and efficient.
\end{abstract}

\section{Introduction}

\label{section_introduction}

In the recent years, due to the rapid progress of data acquisition
and communication technologies, it has become readily easy to
collect and store large amounts of data. Large databases of
scientific measurements at the scale of terabyte or even petabyte
can be frequently observed in high energy physics, astronomy, space
exploration and human genome projects. Large databases of financial
data and sale transactions at the scale of terabyte or petabyte can
also be frequently observed. These huge amounts of data usually
contain valuable scientific and business information. For example, a
large collection of sale transaction data may contain important
information of consumer behaviors and market trends. However, the
data analysis on such large databases presents many technique
challenges. The database size is usually far larger than the memory
size of any single computer. Many existing centralized data analysis
algorithms fail for these instances. In fact, most data analysis
problems for large databases are currently open or not well-solved.

In this paper, we consider one important data analysis problem, the
clustering problem for large databases. The clustering problem is
the problem that a set of given data samples are classified into
different groups, so that, the data samples within each group are
similar according to certain metrics. Clustering is a fundamental
problem in data analysis. It has many applications in pattern
recognition, machine learning, data mining, computer vision, and
signal processing. For example, clustering is usually an important
step in many data mining algorithms.

Many algorithms for clustering problems have been previously
discussed in the literature, see for example \cite{jain99} and
references therein. These algorithms range from heuristic algorithms
to statistical modeling based algorithms. Among the previous
algorithms, the statistical modeling based methods generally have
better clustering performance compared with other types of
algorithms, especially when the data clusters are not well
separated. The Expectation-Maximization (EM) algorithms with mixture
Gaussian modeling \cite{dempster77} \cite{cheeseman96} are the major
state-of-the-art statistical modeling based clustering algorithms.
The EM algorithms can be considered as iterative algorithms for
computing the maximum likelihood estimation. It has been proven that
the likelihood functions do not decrease during iterations.

However, it is well-known that the EM algorithms have certain
limitations. First, according to previous experimental results, the
EM algorithms may convergence very slowly \cite{archambeau03},
\cite{yang98}. It is shown in \cite{xu96}, that the EM algorithms
are first-order optimization algorithms, which provides a
theoretical explanation for the slow convergence speeds. In fact, it
has been a long-standing open problem that super-linear and
second-order methods should be found and preferred for the
clustering problems \cite{redner84}.  Second, the EM algorithms do
not converge and have numerical difficulties for certain types of
instances \cite{archambeau03}, \cite{fraley98}. For example, the EM
algorithms do not converge, when the covariance matrices are
singular. The EM algorithms also do not converge, when the numbers
of components in the mixture modeling are greater than the actual
numbers of data clusters.

In addition, the standard EM algorithms require  memory spaces
proportional to the database size, therefore, do not scale well.
Various scaling-up versions of the standard EM algorithms have been
proposed in the literature \cite{bradley98b}, \cite{zhang96}.
However, these previous approaches are approximation algorithms. The
accuracy of the obtained results decreases as the ratio between the
database size and main processor memory space size increases.


In this paper, we propose a new clustering algorithm for large
databases based on data compression principles and mixture Gaussian
modeling. Following the approaches in \cite{ma09}, we formulate the
clustering problems as optimization problems. Instead of using a
centralized approach, we propose a distributed algorithm to solve
the global optimization problems. In our approach, the global
optimization problem is decomposed into small-scale sub optimization
problems using the Dantzig-Wolfe decomposition method
\cite{dantzig60}. Generally speaking, the Dantzig-wolfe method can
only be used in the convex optimization case, where the duality gaps
are zero. Even though, the considered problem in this paper is
non-convex, we show that the duality gap goes to zeros as the
problem size goes to infinity. Therefore, the Dantzig-Wolfe method
can be applied here. Our algorithm is especially suitable for the
cases of distributed databases, where data are stored at multiple
hosts or even at different geographical locations. The global
optimal solutions can be computed with only intra-host computations,
intra-host local database queries and a small amount of inter-host
communications. Unlike many clustering algorithms for large
databases, which compute approximate solutions, our algorithm
computes exact solutions. Numerical results show that the proposed
algorithm does not have any numerical difficulties for the case that
the covariance matrices are singular. Numerical results also show
that the algorithm has fast convergence speeds.

The rest of this paper is organized as follows. We present the
proposed algorithm in Section
\ref{section_classification_algorithm}.  We prove that the duality
gap is vanishing for sufficiently large databases in Section
\ref{sec_correct}. Numerical results are presented in Section
\ref{sec_numerical}. We present the conclusion remark in Section
\ref{sec_conclusion}.

\emph{Notation:} We use bold face lower-case letters and bold face
capital letters to denote the column vectors and matrices
respectively. For example, we use $\pmb{a}$ to denote a column
vector $\pmb{a}$. We use $\pmb{a}(d)$ to denote the $d$-th element
of the vector $\pmb{a}$. We use  $\pmb{A}^t$ to denote the transpose
of the matrix $\pmb{A}$. We use $H(p_1,\ldots,p_J)$ to denote the
entropy function,
\begin{align}
H(p_1,\ldots,p_J) = \sum_{i=1}^{J} -p_i\log\left(p_i\right).
\end{align}
We use $\log(x)$ to denote the natural logarithm of the number~$x$.
We use $det(\pmb{A})$ to denote the determinant of the matrix
$\pmb{A}$.

\section{Clustering Algorithm}
\label{section_classification_algorithm}

In this paper, we consider a data set consisting of  $N$ data
samples, $ \pmb{x}_1,\pmb{x}_2,\ldots,\pmb{x}_N$, where each data
sample is a $D$ dimensional vector. We assume that the data samples
are randomly distributed with a mixture Gaussian distribution. That
is,
\begin{align}
& p(\pmb{x}_n) = \notag \\
& \sum_{i=1}^{J} p_i
\frac{1}{(2\pi)^{D/2}det(\pmb{\Sigma}_i)^{1/2}}\exp\left\{-\frac{1}{2}(\pmb{x}_n-\pmb{\mu}_i)^t\pmb{\Sigma}_i^{-1}(\pmb{x}_n-\pmb{\mu}_i)\right\}
\end{align}
Alternatively, we may consider $\pmb{x}_1,\ldots, \pmb{x}_n,
\ldots,\pmb{x}_N$ as a mixture of data samples from $J$ information
sources, where each information source is Gaussian distributed. The
considered  problem is therefore estimating the membership of each
data sample to one of the $J$ information sources, and also the
probability distribution of each information source.

In this paper, we propose a distributed algorithm for the above
clustering problem. Our algorithm is efficient for the case that the
data set contains a large amount of data samples. The data samples
can be stored at multiple computers or database hosts. The proposed
algorithm formulates the clustering problem as an optimization
problem and decomposes the optimization problem into multiple
small-scale sub optimization problems. Each sub optimization problem
can be solved at one database host using only locally stored data
samples. A center processor coordinates the computation at the
database hosts. The final solution is obtained from the sub
optimization results. A diagram of the system is shown in
Fig.~\ref{dist_class_block}.

\begin{figure}[h]
 \centering
 \includegraphics[width=3in]{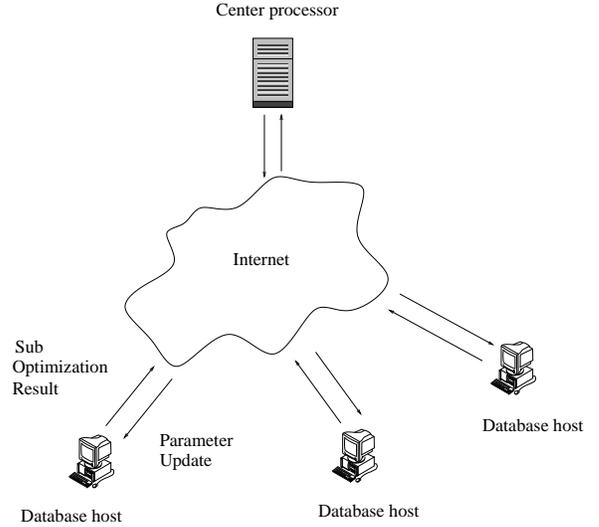}
 \caption{The diagram of the system.}
 \label{dist_class_block}
\end{figure}

The algorithm in this paper is built up on the data compression
based algorithm for clustering in \cite{ma09}. The main idea behind
the algorithm is that optimal data clustering should induce optimal
adaptive data compression. That is, if we partition the data set
into several clusters and use one data compression encoder for each
cluster, then the optimal compression performance is achieved if
each cluster contains only the data samples from  one information
source. The algorithm in \cite{ma09} then formulates the data
cluster problem as an optimization problem, where the classification
gain is maximized. The classification gain is a measure of data
compression efficiency previously proposed in the data compression
literature \cite{joshi97}.

If the covariance matrices of all clusters are not singular, then
the classification gain is inversely proportional to the following
function,
\begin{align}
2H(p_1,\ldots,p_J) + \sum_{i=1}^{J} p_i
\log\left(det\left(\pmb{\Sigma}_i\right)\right)
\end{align}
where, $p_i$ is the fraction of data samples in the $i$-th cluster,
and $\pmb{\Sigma}_i$ is the covariance matrix of the $i$-th cluster.
The above function is the objective function in our optimization
formulation. In the sequel, we will always assume that the
covariance matrices of all clusters are not singular without the
loss of generality. Because, if any covariance matrix is singular,
we can minimize the following function in the algorithm instead,
\begin{align}
2H(p_1,\ldots,p_J) + \sum_{i=1}^{J} p_i
\log\left(det\left(\pmb{\Sigma}_i+\sigma_n^2 \pmb{I}_D\right)\right)
\end{align}
where, $\sigma_n^2$ is a sufficiently small positive number, and
$\pmb{I}_D$ is the $D$ dimensional identity matrix. This is
equivalent to adding white noise with covariance matrix
$\sigma_n^2\pmb{I}_D$ to the data samples and clustering the noise
corrupted data samples instead. The optimality of the final obtained
clustering results is not much affected, if $\sigma_n^2$ is small
enough.

The proposed algorithm formulates the clustering problem as an
optimization problem. We introduce a variable $a_{ni}$ for each
$n,i$, $1\leq n \leq N$, and $1\leq i \leq J$. The variable $a_{ni}$
is a likelihood that the $n$-th data sample belongs to the $i$-th
information source. The mean  $\pmb{\mu}_i$, covariance matrix
$\pmb{\Sigma}_i$, and occurrence probability $p_i$ are functions of
the likelihood variables $a_{ni}$,
\begin{align}
&  \pmb{\mu}_i=\frac{\sum_{n=1}^{N}a_{ni}\pmb{x}_n}{\sum_{n=1}^{N}a_{ni}} \\
& \pmb{\Sigma}_i=
\left(\frac{1}{\sum_{n=1}^{N}a_{ni}}\right)\sum_{n=1}^{N}a_{ni}
(\pmb{x}_n-\pmb{\mu}_i)(\pmb{x}_n-\pmb{\mu}_i)^t
\\
& p_i=\frac{\sum_{n=1}^{N}a_{ni}}{N}.
\end{align}
The formulated optimization problem is therefore,
\begin{align}
\label{original_optimization} &  \min_{\pmb{a}}
\left\{2H(p_1,\ldots,p_J) + \sum_{i=1}^{J} p_i
\log\left(det\left(\pmb{\Sigma}_i\right)\right)\right\} \notag
\\
& \mbox{Subject to: } \pmb{a}\in \Omega
\end{align}
where, $\pmb{a}$ is a vector obtained by stacking all the variables
$a_{ni}$,
\begin{align}
\Omega = \left\{ \pmb{a}\left| \sum_{i=1}^{J} a_{ni}=1,  0\leq
a_{ni}\leq 1\right.\right\}.
\end{align}
The final estimation results can be obtained by randomly rounding
the optimal solution $a_{ni}^\ast$ of the above optimization problem
as in \cite{ma09}. The near-optimality of this optimization based
approach has been shown in \cite{ma09} and \cite{ma10}.

In the sequel, we show that the optimization problem in Eqn.
\ref{original_optimization} can be reduced into sub optimization
problems that can be locally solved at each database host. The
reduction and reformulation procedure consists of four steps.

In the first step of reformulating the problem, we adopt an approach
of first solving the restricted optimization problems with $p_i$
being fixed,
\begin{align}
\label{optimization_with_fix_p}
& g(\widetilde{p}_1,\ldots,\widetilde{p}_J)^\ast \notag \\
& =
\min_{\pmb{a}} \left\{2H(\widetilde{p}_1,\ldots,\widetilde{p}_J) +
\sum_{i=1}^{J} \widetilde{p}_i
\log\left(det\left(\pmb{\Sigma}_i\right)\right)\right\} \notag
\\
& \mbox{Subject to: } \pmb{a}\in \Omega, \mbox{ and
}\sum_{n=1}^{N}a_{ni} = \widetilde{p}_i N,\mbox{ for all }i,
\end{align}
And then, we optimize over $\widetilde{p}_1,\ldots,\widetilde{p}_J$
to find the overall optimization solution,
\begin{align}
\label{optimization_CG} &
\min_{\widetilde{p}_1,\ldots,\widetilde{p}_J}
g(\widetilde{p}_1,\ldots,\widetilde{p}_J)^\ast,  \notag
\\
&\mbox{Subject to: } \sum_{i=1}^{J} \widetilde{p}_i=1, 0\leq
\widetilde{p}_i\leq 1.
\end{align}
The problem in Eqn. \ref{optimization_CG} can be easily solved by
using the gradient descent approach. The main problem is therefore
reduced to the optimization problem in Eqn.
\ref{optimization_with_fix_p}.

In the second step of reformulating the problem, we introduce
auxiliary unitary matrices $\pmb{A}_1,\ldots,\pmb{A}_J$.  We define
$\pmb{B}_i = \pmb{A}_i\pmb{\Sigma}_i\pmb{A}_i^t$, for
$i=1,\ldots,J$. It can be shown that the optimization problem in
Eqn. \ref{optimization_with_fix_p} is equivalent to the following
optimization problem.
\begin{align}
 \label{nd_prob_one} & \min_{\pmb{A}_1,\ldots,\pmb{A}_J,\pmb{a}}
\sum_{i=1}^{J} \widetilde{p}_i \sum_{d=1}^{D}\log\left(
\sigma_{id}^2 \right), \notag \\
&\mbox{Subject to: }\pmb{a}\in \Omega,
\,\,\,\pmb{A}_1,\ldots,\pmb{A}_J
\mbox{ are unitary} \notag \\
& \sum_{n=1}^{N} a_{ni} = \widetilde{p}_i N
\end{align}
where, $\sigma_{id}^2$ is the $d$-th diagonal element of the matrix
$\pmb{B}_i$. The two optimization problems are equivalent, because
\begin{align}
& \sum_{i=1}^{J} \widetilde{p}_i
\log\left(det\left(\pmb{\Sigma}_i\right)\right)
 \leq  \sum_{i=1}^{J} \widetilde{p}_i \sum_{d=1}^{D}\log\left(
\sigma_{id}^2 \right)
\end{align}
due to the Hadamard inequality \cite[page 502, Thm. 16.8.2]{cover},
and clearly equality can be achieved by certain
$\pmb{A}_1,\ldots,\pmb{A}_J$.

We solve the optimization problem in Eqn. \ref{nd_prob_one} by an
alternating optimization approach. That is, we iteratively first fix
$\pmb{A}_1,\ldots,\pmb{A}_J$ and optimize over $\pmb{a}$, and then
fix $\pmb{a}$ and optimize over $\pmb{A}_1,\ldots,\pmb{A}_J$. The
latter optimization problem is easy to solve, because the optimal
$\pmb{A}_1,\ldots,\pmb{A}_J$ are clearly the matrices, such that
$\pmb{B}_i$ becomes diagonal. The main optimization problem is
therefore reduced to
\begin{align}
 \label{nd_prob_two} & \min_{\pmb{a}}
\sum_{i=1}^{J} \widetilde{p}_i \sum_{d=1}^{D}\log\left(
\sigma_{id}^2 \right), \notag \\
&\mbox{Subject to: }\pmb{a}\in \Omega,\,\,\, \sum_{n=1}^{N}a_{ni} =
\widetilde{p}_i N,
\end{align}
where, $\pmb{A}_1,\ldots,\pmb{A}_J$ are fixed and given.

In the third step of reformulating the problem, we use an iterative
upper bounding and minimizing approach to solve the optimization
problem in Eqn.~\ref{nd_prob_two}. Let $\sigma_{id}^{2}[t]$ denote
the solution obtained in the $t$-th iteration. Note that the
objective function in Eqn~\ref{nd_prob_two} can be upper bounded as
follows, due to the fact that the objective function is  concave
with respect to~$\sigma_{id}^{2}$.
\begin{align}
& \sum_{i=1}^{J} \widetilde{p}_i \sum_{d=1}^{D}
\log\left(\sigma_{id}^2\right) \notag \\
& \leq \sum_{i=1}^{J} \widetilde{p}_i \sum_{d=1}^{D}
\log\left(\sigma_{id}^2[t]\right) + \sum_{i=1}^{J}\sum_{d=1}^{D}
\frac{\widetilde{p}_i}{\sigma_{id}^2[t]}\left(\sigma_{id}^2-\sigma_{id}^2[t]\right)
\end{align}
In the $(t+1)$-th iteration, we find a solution $\pmb{a}$, such that
the corresponding $\sigma_{id}^{2}$ minimizes the above upper bound.
It can be seen clearly that the objective function never increase
during iterations. Therefore, the main optimization problem is
reduced to the following optimization problem.
\begin{align}
\label{optimization_beta_introduced} &
  \min_{\pmb{a}} \left\{  \sum_{i=1}^{J}\sum_{d=1}^{D}
\widetilde{p}_i\beta_{id} \sigma_{id}^2\right\} \notag
\\
& \mbox{Subject to: } \pmb{a}\in \Omega,\,\,\, \sum_{n=1}^{N}a_{ni}
= \widetilde{p}_i N,
\end{align}
where $\beta_{id}=1/\sigma_{id}^{2}[t]$.

In the fourth and final step of reformulating the problem, we
decompose the optimization problem in Eqn.
\ref{optimization_beta_introduced} into sub optimization problems by
using the Dantzig-Wolfe decomposition method. Each sub optimization
problem can be locally solved at each database host. The
Dantzig-Wolfe decomposition method is introduced initially for
linear programming problems \cite{dantzig60}. The method has been
then generalized to the convex optimization cases, where the duality
gaps are zero, (see for example \cite{bertsekas99} and references
therein). For non-convex optimization problems, the decomposition
method generally can not be applied due to the non-zero duality
gaps. Even though the optimization problem in Eqn.
\ref{optimization_beta_introduced} is not convex, we show in Theorem
\ref{vanishing_duality_theorem} that the duality gap goes to zeros
as the number of data samples $N$ goes to infinity. Therefore, the
decomposition method can be applied here.

Let us assume that the data samples
$\pmb{x}_1,\ldots,\pmb{x}_n,\ldots,\pmb{x}_N$ are stored at $K$
database hosts. Let ${\mathcal N}_k$ denote the set of the indexes
of the data samples stored at the $k$-th host. We use
$\pmb{A}_i\pmb{x}_n(d)$ to denote the $d$-th element of the vector
$\pmb{A}_i\pmb{x}_n$. The optimization problem in Eqn.
\ref{optimization_beta_introduced} is equivalent to the following
optimization problem.
\begin{align}
\label{dist_opt_primal} & \min_{\pmb{a}, \pmb{\widehat{\mu}}}
\left\{ \sum_{i=1}^{J}  \sum_{k=1}^{K}\sum_{n\in {\mathcal
N}_k}\sum_{d=1}^{D}
\frac{a_{ni}}{N}\beta_{id}\left[\pmb{A}_i\pmb{x}_n(d)-\widehat{\mu}_{ikd}\right]^2\right\}
\notag
\\
& \mbox{Subject to: }  \notag \\
& \sum_{i=1}^{J} a_{ni}=1,\,\,\, 0\leq a_{ni}\leq 1, \,\,\, \sum_{n=1}^{N} a_{ni}/N= \widetilde{p}_{i}, \notag \\
& \widehat{\mu}_{ikd} = \frac{1}{\widetilde{p}_iN} \sum_{n=1}^{N}
a_{ni} \pmb{A}_i \pmb{x}_n(d),
\end{align}
where, $\pmb{\widehat{\mu}}$ is the vector obtained by stacking all
the variables $\widehat{\mu}_{ikd}$. The real number
$\widehat{\mu}_{ikd}$ can be considered as a local guess or
estimation of the mean of $\pmb{A}_i\pmb{x}_n(d)$ at the $k$-th
database host. If all the local guesses are equal, then  the above
objective function is equal to the objective function in Eqn.
\ref{optimization_beta_introduced}.

Because the duality gap is approximately zero as proven in Theorem
\ref{vanishing_duality_theorem}, the optimization problem in Eqn.
\ref{dist_opt_primal} is approximately equivalent to its Lagrangian
dual problem as follows.
\begin{align}
\label{dist_opt_dual} & \max_{\pmb{\lambda}}\min_{\pmb{a},
\pmb{\widehat{\mu}}} \left\{ \sum_{i=1}^{J} \sum_{k=1}^{K}\sum_{n\in
{\mathcal N}_k}\sum_{d=1}^{D}
\frac{a_{ni}}{N}\beta_{id}(\pmb{A}_i\pmb{x}_n(d)-\widehat{\mu}_{ikd})^2\right\} \notag \\
& + \sum_{i=1}^{J}\sum_{k=1}^{K}\sum_{d=1}^{D} \lambda_{\mu ikd}
\left( \widehat{\mu}_{ikd} - \frac{1}{\widetilde{p}_iN}
\sum_{n=1}^{N} a_{ni} \pmb{A}_i \pmb{x}_n(d) \right)
\notag \\
& + \sum_{i=1}^{J}\lambda_{pi}\left[\sum_{n=1}^{N}a_{ni}/N -
\widetilde{p}_i\right] \notag \\
 & \mbox{Subject to: }  \pmb{a}\in \Omega,
\end{align}
where, $\pmb{\lambda}$ denotes the vector obtained by stacking all
variables $\lambda_{\mu ikd}$ and $\lambda_{pi}$. The above
optimization problem is separable and can be rewritten as,
\begin{align}
\label{global_opt} & \max_{\pmb{\lambda}}\sum_{k=1}^{K}
f_k^\ast-\sum_{i=1}^{J}\lambda_{pi}\widetilde{p}_i,
\end{align}
where, each $f_k^\ast$ is the optimization result of one sub
optimization problem. Let $\pmb{a}_k$ denote the vector obtained by
stacking all variables $a_{ni}$ with $n\in {\mathcal N}_k$. Let
$\pmb{\widehat{\mu}}_k$ denote the vector obtained by stacking all
parameters $\widehat{\mu}_{ikd}$, $i=1,\ldots,J$, $d=1,\ldots,D$.
\begin{align}
\label{local_opt} & f_k^\ast = \min_{\pmb{a}_k,
\pmb{\widehat{\mu}}_{k}} \left\{ \sum_{i=1}^{J}  \sum_{n\in
{\mathcal N}_k} \sum_{d=1}^{D}
\frac{a_{ni}}{N}\beta_{id}(\pmb{A}_i\pmb{x}_n(d)-\widehat{\mu}_{ikd})^2\right\} \notag \\
& + \sum_{i=1}^{J} \sum_{d=1}^{D}\lambda_{\mu ikd}
\widehat{\mu}_{ikd} + \sum_{i=1}^{J}\lambda_{pi}\sum_{n\in {\mathcal
N}_k}\frac{a_{ni}}{N} \notag
\\
& - \sum_{i=1}^{J}\sum_{k=1}^{K}\sum_{d=1}^{D}\frac{\lambda_{\mu
ikd}}{\widetilde{p}N}
\sum_{n\in {\mathcal N}_k}a_{ni}\pmb{A}_i\pmb{x}_n(d) \notag \\
 &
\mbox{Subject to: } \sum_{i=1}^{J} a_{ni}=1,\,\,\, 0\leq a_{ni}\leq
1, \mbox{ for }n\in {\mathcal N}_k.
\end{align}
It can be clearly checked that each $f_k^\ast$ can be solved locally
at each database host using only information about local data
samples $\pmb{x}_n$, $n\in {\mathcal N}_k$ with given  parameters
$\beta_{id}$, $\pmb{\lambda}$, and $\pmb{A}_1,\ldots,\pmb{A}_J$.

Therefore, the proposed algorithm iteratively computes the
clustering result. During each iteration, each database host solves
one local small-scale optimization problem as in Eqn.
\ref{local_opt}. The center processor then solves the global
optimization problem as in Eqn. \ref{global_opt} using the local
optimization results. The global optimization problem can be solved
by using, for example, the subgradient method \cite[Section
6.3.1]{bertsekas99}.

\section{Vanishing Duality Gap}
\label{sec_correct}

In this section, we  prove that the duality gap between the primal
optimization problem in Eqn. \ref{dist_opt_primal} and the dual
optimization problem in Eqn. \ref{dist_opt_dual} goes to zero as the
problem size $N$ goes to infinity. We need the Azuma inequality in
our discussion. A proof of the inequality can be found, for example
in \cite{azuma67}\cite{janson98}.
\begin{lemma}
\label{azuma_inequality}\emph{(Azuma Inequality)} Let
$Z_1,\dots,Z_N$ be independent random variables, with $Z_k$ taking
values in a set $\Lambda_k$. Assume that a (measurable) function
$f:\Lambda_1\times \Lambda_2\times \cdots \times
\Lambda_N\rightarrow {\mathbb R}$ satisfies the following Lipschitz
condition (L).
\begin{itemize}
\item (L) If the vectors $z,z'\in\prod_{1}^{N}\Lambda_i$ differ only
in the $k$th coordinate, then $|f(z)-f(z')|<c_k$, $k=1,\ldots,N$.
\end{itemize}
Then, the random variable $X=f(Z_1,\ldots,Z_N)$ satisfies, for any
$t\geq 0$,
\begin{align}
& {\mathbb P}(X\geq {\mathbb E}X+t)\leq
\exp\left(\frac{-2t^2}{\sum_{1}^{N}c_k^2}\right),
\\
& {\mathbb P}(X\leq {\mathbb E}X-t)\leq
\exp\left(\frac{-2t^2}{\sum_{1}^{N}c_k^2}\right).
\end{align}
\end{lemma}

The basic idea is to use randomization. Randomization has been used
previously  in  establishing stronger duality theories. We refer
interested readers to \cite{ermoliev85} and references therein. Let
 $p(\pmb{a}, \pmb{\widehat{\mu}})$ denote the probability distribution of
$\pmb{a}$ and $\pmb{\widehat{\mu}}$, where the range of $\pmb{a}$ is
$\Omega$, and
\begin{align}\min_n \pmb{A}_i\pmb{x}_n(d)\leq
\widehat{\mu}_{ikd}\leq \max_n \pmb{A}_i\pmb{x}_n(d).
\end{align} We
introduce the following randomized primal optimization problem.
\begin{align}
\label{dist_opt_primal_random} & \min_{p(\pmb{a},
\pmb{\widehat{\mu}})} {\mathbb E} \left\{ \sum_{i=1}^{J}
\sum_{k=1}^{K}\sum_{n\in {\mathcal N}_k}\sum_{d=1}^{D}
\frac{a_{ni}}{N}\beta_{id}(\pmb{A}_i\pmb{x}_n(d)-\widehat{\mu}_{ikd})^2\right\} \notag \\
& \mbox{Subject to: }  \notag \\
& {\mathbb E}\left[ \left( \widehat{\mu}_{ikd} -
\frac{1}{\widetilde{p}_iN} \sum_{n=1}^{N} a_{ni} \pmb{A}_i
\pmb{x}_n(d) \right)\right] = 0,
\notag \\
& {\mathbb E} \left[\left.\frac{1}{N} \sum_{n=1}^{N}a_{ni} -
\widetilde{p}_{i}\right|\pmb{\widehat{\mu}}\right] = 0, \mbox{ for
all }\pmb{\widehat{\mu}}.
\end{align}
The corresponding Lagrangian randomized dual problem is
\begin{align}
\label{dist_opt_dual_random} & \max_{\pmb{\lambda}} \min_{p(\pmb{a},
\pmb{\widehat{\mu}})} {\mathbb E} \left\{ \sum_{i=1}^{J}
\sum_{k=1}^{K}\sum_{n\in {\mathcal N}_k}\sum_{d=1}^{D}
\frac{a_{ni}}{N}\beta_{id}(\pmb{A}_i\pmb{x}_n(d)-\widehat{\mu}_{ikd})^2\right\} \notag \\
& + \int \sum_{i=1}^{J}\lambda_{pi}(\pmb{\widehat{\mu}}) {\mathbb E}
\left[\left.\frac{1}{N}
\sum_{n=1}^{N}a_{ni} - \widetilde{p}_{i}\right|\pmb{\widehat{\mu}}\right] d\pmb{\widehat{\mu}} \notag \\
& + \sum_{i=1}^{J}\sum_{k=1}^{K}\sum_{d=1}^{D} \lambda_{\mu ikd}
{\mathbb E} \left[ \widehat{\mu}_{ikd} - \frac{1}{\widetilde{p}_iN}
\sum_{n=1}^{N} a_{ni} \pmb{A}_i \pmb{x}_n(d) \right].
\end{align}

Let us denote the optimal solutions of the primal optimization
problem in Eqn. \ref{dist_opt_primal}, randomized primal
optimization problem in Eqn. \ref{dist_opt_primal_random}, dual
optimization problem in Eqn. \ref{dist_opt_dual}, and  randomized
dual optimization problem in Eqn. \ref{dist_opt_dual_random} by
$P^\ast$, $PR^\ast$, $D^\ast$, and $DR^\ast$ respectively. We have
the following lemmas.

\begin{lemma}
\label{lemma_main_one}
\begin{align}
PR^\ast \leq P^\ast
\end{align}
\end{lemma}
\begin{proof}
The lemma follows from the fact that each deterministic variable can
be considered as a random variable with a singleton probability
distribution.
\end{proof}

\begin{lemma}
\label{lemma_main_two}
\begin{align}
 DR^\ast \leq D^\ast
\end{align}
\end{lemma}
\begin{proof}
Similar as the proof of Lemma \ref{lemma_main_one}.
\end{proof}

\begin{lemma}
\label{lemma_main_three}
\begin{align}
PR^\ast = DR^\ast
\end{align}
\end{lemma}
\begin{proof}
We may define the following $PR_\epsilon$ optimization problem.
\begin{align}
& \min_{p(\pmb{a}, \pmb{\widehat{\mu}})} {\mathbb E} \left\{
\sum_{i=1}^{J}  \sum_{k=1}^{K}\sum_{n\in {\mathcal
N}_k}\sum_{d=1}^{D}
\frac{a_{ni}}{N}\beta_{id}(\pmb{A}_i\pmb{x}_n(d)-\widehat{\mu}_{ikd})^2\right\} \label{pr_epsilon_obj_fun} \\
& \mbox{Subject to: }   \notag \\
& \left|{\mathbb E} \left[\left.\frac{1}{N} \sum_{n=1}^{N}a_{ni} -
\widehat{p}_{i}\right|\pmb{\widehat{\mu}}\right]\right|  \leq \epsilon,  \mbox{ for all }\pmb{\widehat{\mu}}, \label{pr_epsilon_con_one}  \\
&  \left| \widehat{\mu}_{ikd} - \frac{1}{\widetilde{p}_iN}
\sum_{n=1}^{N} a_{ni} \pmb{A}_i \pmb{x}_n(d) \right|\leq \epsilon.
\label{pr_epsilon_con_two}
\end{align}
It can be check that $PR_\epsilon^\ast \leq PR^\ast$, and
$PR_\epsilon^\ast \rightarrow PR^\ast$, as $\epsilon\rightarrow 0$.
The dual of the $PR_\epsilon$ problem $DR_\epsilon$ is
\begin{align}
 & \max_{\pmb{\lambda}^-, \pmb{\lambda}^+} \min_{p(\pmb{a},
\pmb{\widehat{\mu}})} {\mathbb E} \left\{ \sum_{i=1}^{J}
\sum_{k=1}^{K}\sum_{n\in {\mathcal N}_k}\sum_{d=1}^{D}
\frac{a_{ni}}{N}\beta_{id}(\pmb{A}_i\pmb{x}_n(d)-\widehat{\mu}_{ikd})^2\right\} \notag \\
& + \sum_{i=1}^{J}\sum_{k=1}^{K}\sum_{d=1}^{D}  \lambda_{\mu ikd}^-
\left( {\mathbb E}\left[\widehat{\mu}_{ikd} -
\frac{1}{\widetilde{p}_iN} \sum_{n=1}^{N} a_{ni} \pmb{A}_i
\pmb{x}_n(d)\right] - \epsilon \right)
\notag \\
& + \sum_{i=1}^{J}\sum_{k=1}^{K}\sum_{d=1}^{D}  (-1)\lambda_{\mu
ikd}^+ \left( {\mathbb E}\left[\widehat{\mu}_{ikd} -
\frac{1}{\widetilde{p}_iN} \sum_{n=1}^{N} a_{ni} \pmb{A}_i
\pmb{x}_n(d)\right] + \epsilon \right)
\notag \\
& + \int \sum_{i=1}^{J}\lambda_{pi}^-(\pmb{\widehat{\mu}})
\left\{{\mathbb E} \left[\left.\frac{1}{N}
\sum_{n=1}^{N} a_{ni} - \widetilde{p}_{i}\right|\pmb{\widehat{\mu}}\right]-\epsilon \right\} d\pmb{\widehat{\mu}}\notag \\
& + \int \sum_{i=1}^{J}(-1)\lambda_{pi}^+(\pmb{\widehat{\mu}})
\left\{{\mathbb E} \left[\left.\frac{1}{N}
\sum_{n=1}^{N}a_{ni} - \widehat{p}_{i}\right|\pmb{\widehat{\mu}}\right]+\epsilon\right\} d\pmb{\widehat{\mu}} \notag \\
& \mbox{Subject to: }  \lambda_{\mu ikd}^-\geq 0, \lambda_{\mu
ikd}^+\geq 0,\lambda_{pi}^-(\pmb{\widehat{\mu}})\geq 0,
\lambda_{pi}^+(\pmb{\widehat{\mu}})\geq 0.
\end{align}
It can be also checked that $DR_\epsilon^\ast \rightarrow DR^\ast$,
as $\epsilon\rightarrow 0$.

Now we show that $PR_\epsilon$ is a convex optimization problem. Let
$p^1(\pmb{a},\pmb{\widehat{\mu}})$,
$p^2(\pmb{a},\pmb{\widehat{\mu}})$ be two probability distributions
satisfying all the constraints in the $PR_{\epsilon}$ problem. Let
\begin{align}
p(\pmb{a},\pmb{\widehat{\mu}}) = \alpha
p^1(\pmb{a},\pmb{\widehat{\mu}}) +
(1-\alpha)p^2(\pmb{a},\pmb{\widehat{\mu}}),
\end{align}
where, $0\leq \alpha \leq 1$. Equivalently, we may introduce a
random variable $z$, ${\mathbb P}(z=1)=\alpha$,  ${\mathbb
P}(z=2)=1-\alpha$; $p(\pmb{a},\pmb{\widehat{\mu}}) =
p^1(\pmb{a},\pmb{\widehat{\mu}}) $, if $z=1$, and
$p(\pmb{a},\pmb{\widehat{\mu}}) = p^2(\pmb{a},\pmb{\widehat{\mu}})
$, if $z=2$. We can show that $p(\pmb{a},\pmb{\widehat{\mu}})$
satisfies the constraint in Eqn. \ref{pr_epsilon_con_one} as
follows.
\begin{align}
& {\mathbb
E}\left[\left.\sum_{n=1}^{N}\frac{a_{ni}}{N}-\widehat{p}_i\right|\pmb{\widehat{\mu}}\right]
= \int \left[\sum_{n=1}^{N}\frac{a_{ni}}{N}-\widehat{p}_i\right]
p(\pmb{a}|\pmb{\widehat{\mu}}) d\pmb{a} \notag \\
& = \int \left[\sum_{n=1}^{N}\frac{a_{ni}}{N}-\widehat{p}_i\right]
p(\pmb{a},z=1|\pmb{\widehat{\mu}})  d\pmb{a} \notag \\
& \hspace{0.2in} + \int
\left[\sum_{n=1}^{N}\frac{a_{ni}}{N}-\widehat{p}_i\right]
p(\pmb{a},z=2|\pmb{\widehat{\mu}})  d\pmb{a} \notag \\
& = \int \left[\sum_{n=1}^{N}\frac{a_{ni}}{N}-\widehat{p}_i\right]
p^1(\pmb{a}|\pmb{\widehat{\mu}})  p(z=1|\pmb{\widehat{\mu}}) d\pmb{a} \notag \\
& \hspace{0.2in} + \int
\left[\sum_{n=1}^{N}\frac{a_{ni}}{N}-\widehat{p}_i\right]
p^2(\pmb{a}|\pmb{\widehat{\mu}})  p(z=2|\pmb{\widehat{\mu}}) d\pmb{a} \notag \\
& \leq   p(z=1|\pmb{\widehat{\mu}}) \epsilon +
p(z=2|\pmb{\widehat{\mu}}) \epsilon \,\,\,\leq \epsilon
\end{align}
Similarly,
\begin{align}
& {\mathbb
E}\left[\left.\sum_{n=1}^{N}\frac{a_{ni}}{N}-\widehat{p}_i\right|\pmb{\widehat{\mu}}\right]
\geq \epsilon
\end{align}
We can also show that $p(\pmb{a},\pmb{\widehat{\mu}})$ satisfies the
 constraint in Eqn. \ref{pr_epsilon_con_two} by using the fact that the expectation is a linear functional.
Finally,  the objective function in Eqn. \ref{pr_epsilon_obj_fun} is
also convex, because the expectation is a linear functional.
Therefore, the optimization problem $PR_\epsilon$ is a convex
optimization problem.

Because, $PR_\epsilon$ is a convex optimization problem and the
Slater condition holds, $PR_\epsilon^\ast = DR_\epsilon^\ast$
according to the strong duality theorem \cite[Thm. 6.7]{jahn07}.
Therefore, $PR^\ast = DR^\ast$.
\end{proof}

\begin{lemma}
\label{lemma_main_four} Assume $\max_{n,m} ||\pmb{x}_n -
\pmb{x}_m||_2\leq V$, for a fixed upper bound $V$, where
$||\cdot||_2$ denotes the Euclidean norm. Then $ PR^\ast \rightarrow
P^\ast $, as  $N$ goes to infinity.
\end{lemma}
\begin{proof}
Let $p^\ast(\pmb{a},\pmb{\widehat{\mu}})$ denote the optimal
solution of the randomized primal problem. We can construct a
probability distribution $\widehat{p}(\pmb{a},\pmb{\widehat{\mu}})$
as follows.
\begin{align}
\label{hat_probability_distribution}
\widehat{p}(\pmb{a},\pmb{\widehat{\mu}}) =
p^\ast(\pmb{\widehat{\mu}})
\prod_{n=1}^{N}p^\ast(a_{n1},\ldots,a_{nJ}|\pmb{\widehat{\mu}}),
\end{align}
where, the probability distributions at the right hand  are marginal
distributions. It can be checked that the probability
$\widehat{p}(\pmb{a},\pmb{\widehat{\mu}})$ achieves the exactly same
objective function and constraint function values in the randomized
primal problem as the probability distribution
$p^\ast(\pmb{a},\pmb{\widehat{\mu}})$. Therefore, we can assume that
$p^\ast(\pmb{a},\pmb{\widehat{\mu}})$ takes the form in Eqn.
\ref{hat_probability_distribution} without the loss of generality.

We define the typical set ${\mathcal T}(\epsilon)$ as
\begin{align}
{\mathcal T}(\epsilon) =
\left\{(\pmb{a},\pmb{\widehat{\mu}})\left|\left|\sum_{n=1}^{N}\frac{a_{ni}}{N}
- \widetilde{p}_{i} \right|\leq \epsilon, \mbox{ for all
}i\right.\right\}.
\end{align}
The probability that $(\pmb{a}, \pmb{\widehat{\mu}})$ is not in the
typical set ${\mathcal T}(\epsilon)$ can be upper bounded by using
the Azuma inequality and the union bound as follows.
\begin{align}
& {\mathbb P}\left[(\pmb{a},\pmb{\widehat{\mu}})\notin {\mathcal
T}(\epsilon)\right] \leq \sum_{i=1}^{J}{\mathbb
P}\left[\left|\sum_{n=1}^{N}\frac{a_{ni}}{N} - \widetilde{p}_{i}
\right|\geq \epsilon\right] \notag \\
& \leq \sum_{i=1}^{J}\int {\mathbb
P}\left[\left.\left|\sum_{n=1}^{N}\frac{a_{ni}}{N} -
\widetilde{p}_{i}
\right|\geq \epsilon\right|\pmb{\widehat{\mu}}\right] p(\pmb{\widehat{\mu}})d\pmb{\widehat{\mu}}\notag \\
& \leq  \sum_{i=1}^{J} \int 2\exp\left(-2\epsilon^2N\right) p(\pmb{\widehat{\mu}})d\pmb{\widehat{\mu}} \notag \\
& \leq  2J\exp\left(-2\epsilon^2N\right)
\end{align}

Due to the fact that the objective function is non-negative, the
average achieved objective function values by
$(\pmb{a},\pmb{\widehat{\mu}})$ in the typical set,
\begin{align}
\label{azuma_condition_one} & {\mathbb E}\left\{\left.
\sum_{i=1}^{J}  \sum_{k=1}^{K}\sum_{n\in {\mathcal
N}_k}\sum_{d=1}^{D}
\frac{\widehat{a}_{ni}}{N}\beta_{id}(\pmb{A}_i\pmb{x}_n(d)-\widehat{\mu}_{ikd})^2\right|{\mathcal T}(\epsilon)\right\} \notag \\
& \leq \frac{PR^\ast} {{\mathbb P}((\pmb{a},\pmb{\widehat{\mu}})\in
{\mathcal T}(\epsilon))}
\end{align}
Also by the above discussions,
\begin{align}
& {\mathbb P}\left[(\pmb{a},\pmb{\widehat{\mu}})\in {\mathcal
T}(\epsilon)\right] \geq  1-2J\exp\left(-2\epsilon^2N\right)
\end{align}
Therefore, we have that the average of the objective function in the
typical set is bounded by
\begin{align}
 & {\mathbb E}\left\{\left. \sum_{i=1}^{J}
\sum_{k=1}^{K}\sum_{n\in {\mathcal N}_k}\sum_{d=1}^{D}
\frac{\widehat{a}_{ni}}{N}\beta_{id}(\pmb{A}_i\pmb{x}_n(d)-\widehat{\mu}_{ikd})^2\right|{\mathcal T}(\epsilon)\right\} \notag \\
& \leq \frac{PR^\ast} {1-2J\exp(-2\epsilon^2N)}
\end{align}

There must exist one $(\pmb{\widehat{a}},\pmb{\bar{\mu}})$ in the
typical set, such that the corresponding objective function is less
than or equal to the above average. We can further modify the above
$\pmb{\widehat{a}}$ into a certain $\pmb{\widetilde{a}}\in \Omega$,
$\pmb{\widetilde{a}}= (\ldots,\widetilde{a}_{ni},\ldots)$, such that
\begin{align} \sum_{n=1}^{N}\widetilde{a}_{ni}/N =
\widetilde{p}_{i},\end{align} and the corresponding objective
function is raised by at most $(J-1)\max\{\beta_{id}\}V^2 \epsilon$.
We can now set
\begin{align}\widetilde{\mu}_{ikd} =
\frac{1}{\widetilde{p}_iN}\sum_{n=1}^{N}\widetilde{a}_{ni}\pmb{A}_i
\pmb{x}_n(d).
\end{align}
Clearly, $\widetilde{a}_{ni}$ and $\widetilde{\mu}_{ikd}$ satisfy
all the constraints in the primal problem. Therefore,
\begin{align}
 P^\ast & {\leq}  \sum_{i=1}^{J} \sum_{k=1}^{K}\sum_{n\in
{\mathcal N}_k}\sum_{d=1}^{D}
\frac{\widetilde{a}_{ni}}{N}\beta_{id}\left[(\pmb{A}_i\pmb{x}_n(d)-\widetilde{\mu}_{ikd})^2\right]\notag
\\
& \stackrel{(a)}{\leq} \sum_{i=1}^{J} \sum_{k=1}^{K}\sum_{n\in
{\mathcal N}_k} \sum_{d=1}^{D}
\frac{\widetilde{a}_{ni}}{N}\beta_{id}(\pmb{A}_i\pmb{x}_n(d)-\bar{\mu}_{ikd})^2 \notag \\
& \leq \frac{PR^\ast}{1-2J\exp(-2\epsilon^2N) }  +
(J-1)\max\{\beta_i\}V^2  \epsilon
\end{align}
where, (a) follows from the fact that $\widetilde{\mu}_{ikd}$ are
the minimizer of the above quadratic function. The lemma then
follows from the fact that $PR^\ast \leq P^\ast$.
\end{proof}

\begin{theorem}
\label{vanishing_duality_theorem} The duality gap $P^\ast - D^\ast$
between the primal problem and dual problem goes to zero as the data
sample number $N$ goes to infinity.
\end{theorem}


%

\section{Numerical Results }

\label{sec_numerical}

In this section, we present numerical results for the proposed
clustering algorithm. In Fig. \ref{result_normal_covar}, we depict
the result of the proposed algorithm for the case of two overlapping
clusters in a two dimensional space. Both the two clusters have zero
mean. Their covariance matrices are as follows.
 \begin{align}
 \left[ \begin{array}{cc} 80000  &  52000 \\
52000 & 35600 \end{array}\right],\,\,\, \left[\begin{array}{cc}
     192800  &  -118800 \\
    -118800  &    74000 \end{array}\right].
\end{align}
The total data sample number is $2048$ and each cluster contains
$1024$ data samples. We assume that the data samples can be observed
by two database hosts, where the first database host can only
observe the $1024$ data samples from the first cluster, and the
second database host can only observe the $1024$ data samples from
the second cluster. After the clustering result is obtained, we
randomly select $128$ data samples from the first cluster and $128$
data samples from the second cluster and plot these data samples in
the figure. The data samples classified into one cluster are plotted
as red circles and the data sample classified into the other cluster
are plotted as blue squares. The percentage of missed classified
data samples is $5.32\%$. The clustering errors mainly occur at the
regions where the two clusters overlap. The algorithm starts with
two randomly selected unitary matrices $\pmb{A}_1$, and $\pmb{A}_2$.
We observe that these matrices converge quickly. We also experiment
with the cases that each database host observes a mixture of data
samples from the two clusters with various percentages. The obtained
results are not significantly different from the result in Fig.
\ref{result_normal_covar}.

\begin{figure}[h]
 \centering
 \includegraphics[width=3in]{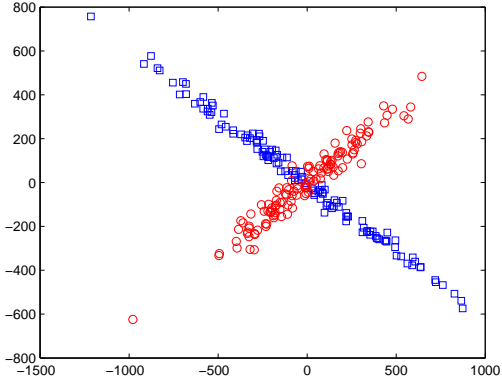}
 \caption{Clustering results for two overlapping clusters.}
 \label{result_normal_covar}
\end{figure}

In Fig. \ref{result_singular_covar}, we depict the result of the
proposed algorithm for the case of two overlapping clusters with one
cluster having a singular covariance matrix. Both the two clusters
have zero mean. Their covariance matrices are as follows.
 \begin{align}
 \left[ \begin{array}{cc} 80000  &  52000 \\
52000 & 35600 \end{array}\right],\,\,\, \left[\begin{array}{cc}
     192800  &  0 \\
    0  &    0 \end{array}\right].
\end{align}
The total data sample number is $2048$ and each cluster contains
$1024$ data samples. There are two database hosts, and the first
database host can only observe the $1024$ data samples from the
first cluster, and the second database host can only observe the
$1024$ data samples from the second cluster. In the formulated
optimization problem, a term $\sigma_n^2\pmb{I}_2$,
$\sigma_n^2=0.5$, is added to the objective function. The clustering
results of randomly selected $256$ data samples are shown in the
figure. The percentage of missed classified data samples is
$1.71\%$. The results for the cases that each database host observes
a mixture of data samples from the two clusters with various
percentages are not significantly different from the result in the
figure. The proposed clustering algorithm does not have any
numerical or convergence difficulties for these cases.

\begin{figure}[h]
 \centering
 \includegraphics[width=3in]{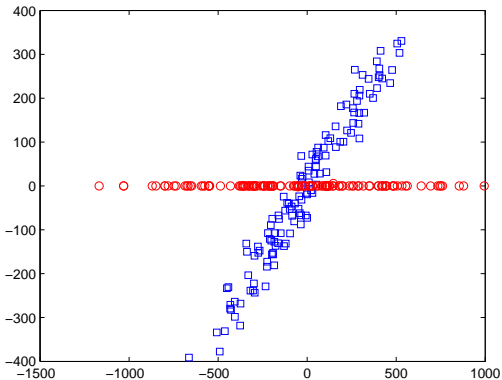}
 \caption{Clustering result for the case that one cluster has a singular covariance matrix.}
 \label{result_singular_covar}
\end{figure}

In Fig. \ref{result_separate}, we depict the result of the proposed
algorithm for the case of two clusters with different means. The
first cluster has zero mean and covariance matrix
 \begin{align}
 \left[ \begin{array}{cc} 80000  &  52000 \\
52000 & 35600 \end{array}\right].
\end{align}
The second cluster has mean $[800,800]^t$ and covariance matrix
 \begin{align}
\left[\begin{array}{cc}
     192800  &  -118800 \\
    -118800  &    74000 \end{array}\right].
\end{align}
The total data sample number is $2048$ and each cluster contains
$1024$ data samples. There are two database hosts, the first
database host can only observe the $1024$ data samples from the
first cluster, and the second database host can only observe the
$1024$ data samples from the second cluster. The percentage of
missed classified data samples is $2.29\%$. The results for the
cases that each database host observes a mixture of data samples
from the two clusters with various percentages are not significantly
different from the result in the figure.

\begin{figure}[h]
 \centering
 \includegraphics[width=3in]{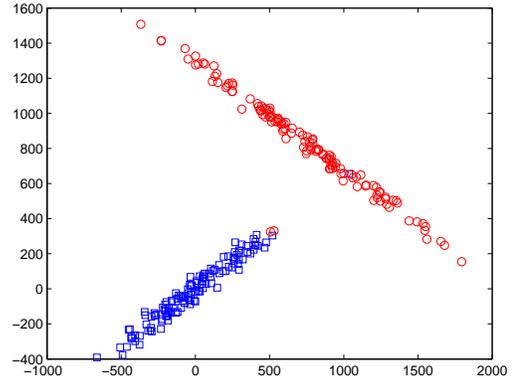}
 \caption{Clustering result for the case that the two clusters have different means.}
 \label{result_separate}
\end{figure}

In summary, we find that the proposed clustering algorithm has low
missed classification probability and fast convergence speeds. The
algorithm does not have numerical or convergence difficulties for
the case of singular covariance matrices. The proposed algorithm is
a promising approach for future large-scale data analysis.

\section{Conclusion}

\label{sec_conclusion}

This paper proposes a large-scale data clustering algorithm based on
distributed optimization. We show that the duality gap of the
considered optimization problem goes to zero as the problem size
goes to infinity. Therefore, the global optimization problem can be
decomposed into small-scale sub optimization problems by using the
Dantzig-Wolfe method. The small-scale sub optimization problems can
be solved using a group of computers coordinated by one center
processor. Numerical results show that the proposed algorithm is
effective, efficient and does not have numerical or convergence
difficulties.

%

\bibliographystyle{IEEEtran}
\bibliography{distributive_signal_classification}

\IEEEpeerreviewmaketitle

\comment{
\section{Introduction}
This demo file is intended to serve as a ``starter file"
for IEEE conference papers produced under \LaTeX\ using IEEEtran.cls version
1.6b and later.

 May all your publication endeavors be successful.

\hfill mds

\hfill November 18, 2002

\subsection{Subsection Heading Here}
Subsection text here.

\subsubsection{Subsubsection Heading Here}
Subsubsection text here.

\section{Conclusion}
The conclusion goes here.


\section*{Acknowledgment}
The authors would like to thank...



%

}

\end{document}